\documentclass[letterpaper, 10 pt, conference]{ieeeconf}  % Comment this line out if you need a4paper
\pdfoutput=1
\IEEEoverridecommandlockouts   
\usepackage{amsmath,amsfonts,amssymb}
\let\labelindent\relax
\usepackage{prettyref}
\usepackage{mathrsfs}
\usepackage{graphicx}
\usepackage{wrapfig}
\usepackage{subfig}
\usepackage{MnSymbol}
\usepackage{multirow}
\usepackage{booktabs}
\usepackage{algorithm}
\usepackage[table]{xcolor}
\usepackage[noend]{algpseudocode}
\usepackage{enumitem}
\usepackage{comment}
\usepackage
[backend=bibtex,
bibstyle=ieee,
citestyle=numeric,
mincitenames=1,
maxcitenames=2,
natbib=true,
doi=false,
isbn=false,
url=false,
eprint=false]{biblatex}
\usepackage[colorlinks,allcolors=gray,hypertexnames=true]{hyperref} 
\newcommand{\citewithauthor}[1]{\citeauthor{#1} \cite{#1}}
\usepackage{diagbox}

\usepackage{pgf}
\pgfkeys{/pgf/number format/.cd ,precision=2,sci generic={exponent={\times 10^{#1}}}}
\newcommand\convert[1]{\pgfmathprintnumber{#1}}

\newtheorem{theorem}{Theorem}[section]
\newtheorem{lemma}[theorem]{\TE{Lemma}}

\algnewcommand{\LineComment}[1]{\State \(\triangleright\) #1}
\algdef{SE}[DOWHILE]{Do}{doWhile}{\algorithmicdo}[1]{\algorithmicwhile\ #1}
\newcommand*{\colorboxed}{}
\def\colorboxed#1#{%
  \colorboxedAux{#1}%
}
\newcommand*{\colorboxedAux}[3]{%
  \begingroup
    \colorlet{cb@saved}{.}%
    \color#1{#2}%
    \boxed{%
      \color{cb@saved}%
      #3%
    }%
  \endgroup
}
\newrefformat{fig}{Figure~\ref{#1}}
\newrefformat{par}{Section~\ref{#1}}
\newrefformat{appen}{Appendix~\ref{#1}}
\newrefformat{sec}{Section~\ref{#1}}
\newrefformat{sub}{Section~\ref{#1}}
\newrefformat{table}{Table~\ref{#1}}
\newrefformat{ass}{Assumption~\ref{#1}}
\newrefformat{alg}{Algorithm~\ref{#1}}
\newrefformat{def}{Definition~\ref{#1}}
\newrefformat{thm}{Theorem~\ref{#1}}
\newrefformat{lem}{Lemma~\ref{#1}}
\newrefformat{step}{Step~\ref{#1}}
\newrefformat{ln}{Line~\ref{#1}}
\newrefformat{eq}{Equation~\ref{#1}}
\newrefformat{pb}{Problem~\ref{#1}}
\newrefformat{it}{Item~\ref{#1}}
\newrefformat{te}{Term~\ref{#1}}
\def\Eqref Eq:#1:{\eqref{eq:#1}}
\newrefformat{Eq}{Equation~\Eqref#1:}

\newcommand{\TE}[1]{\textbf{#1}}

\newcommand{\FPP}[2]{\frac{\partial{#1}}{\partial{#2}}}
\newcommand{\FPPR}[2]{{\partial{#1}}/{\partial{#2}}}

\newcommand{\TWO}[2]{\left(\setlength{\arraycolsep}{1pt}\begin{array}{cc}{#1}, & {#2}\end{array}\right)}
\newcommand{\TWOC}[2]{\left(\setlength{\arraycolsep}{1pt}\begin{array}{c}#1 \\ #2\end{array}\right)}

\newcommand{\THREE}[3]{\left(\setlength{\arraycolsep}{1pt}\begin{array}{ccc}{#1}, & {#2}, & {#3}\end{array}\right)}

\newcommand{\FOURC}[4]{\left(\setlength{\arraycolsep}{1pt}\begin{array}{c}#1 \\ #2 \\ #3 \\ #4\end{array}\right)}

\newcommand{\fmin}[1]{\underset{#1}{\min}}
\newcommand{\fmax}[1]{\underset{#1}{\max}}
\newcommand{\argmin}[1]{\underset{#1}{\text{argmin}}}
\newcommand{\argminP}[1]{\text{argmin}}
\newcommand{\argmax}[1]{\underset{#1}{\text{argmax}}}
\newcommand{\argmaxP}[1]{\text{argmax}}
\newcommand{\ST}{\text{s.t.}}

\newcommand{\proofread}[1]{}
\newif\ifArxiv
\usepackage{xcolor}
\definecolor{Blue}{rgb}{0.2, 0.2, 0.8}

\AtBeginBibliography{\footnotesize}
\makeatletter
\IEEEtriggercmd{\reset@font\normalfont\fontsize{7.0pt}{7.2pt}\selectfont}
\makeatother
\IEEEtriggeratref{1}

\makeatletter
\newcommand\fs@ruled@notop{\def\@fs@cfont{\bfseries}\let\@fs@capt\floatc@ruled
  \def\@fs@pre{}%
  \def\@fs@post{\kern2pt\hrule\relax}%
  \def\@fs@mid{\kern2pt\hrule\kern2pt}%
  \let\@fs@iftopcapt\iftrue}
\renewcommand\fst@algorithm{\fs@ruled@notop}
\makeatother

\title{\large\bf Planning of Power Grasps Using Infinite Program Under Complementary Constraints \vspace{-15px}}
\author{Zherong Pan$^1$, Duo Zhang$^2$, Changhe Tu$^2$, Xifeng Gao$^3$\\
\vspace{-15px}
\thanks{\footnotesize{$^1$ Lightspeed \& Quantum Studio, Tencent America. $\{$zherong.pan.usa, gxf.xisha$\}$@gmail.com. $^2$Department of Computer Science, Shandong University. galenzhang@mail.sdu.edu.cn. Duo Zhang and Changhe Tu are supported in part by NSFC No. 61772318.}}
}
\allowdisplaybreaks
\begin{document}
\maketitle
\thispagestyle{empty}
\pagestyle{empty}

\newif\ifdetail
\detailfalse

%%%%%%%%%%%%%%%%%%%%%%%%%%%%%%%%%%%%%%%%%%%%%%%%%%%%%%%%%%%%%%%%%%%%%%%%%%%%%%%%
\begin{abstract}
We propose an optimization-based approach to plan power grasps. Central to our method is a reformulation of grasp planning as an infinite program under complementary constraints (IPCC), which allows contacts to happen between arbitrary pairs of points on the object and the robot gripper. We show that IPCC can be reduced to a conventional finite-dimensional nonlinear program (NLP) using a kernel-integral relaxation. Moreover, the values and Jacobian matrices of the kernel-integral can be evaluated efficiently using a modified Fast Multipole Method (FMM). We further guarantee that the planned grasps are collision-free using primal barrier penalties. We demonstrate the effectiveness, robustness, and efficiency of our grasp planner on a row of challenging 3D objects and high-DOF grippers, such as Barrett Hand and Shadow Hand, where our method achieves superior grasp qualities over competitors.
\end{abstract}

%%%%%%%%%%%%%%%%%%%%%%%%%%%%%%%%%%%%%%%%%%%%%%%%%%%%%%%%%%%%%%%%%%%%%%%%%%%%%%%%
\section{Introduction}
Grasp planning remains a fundamental and perennial problem, although intense research efforts have been invested over the past decades. A vast majority of prior works view grasp planning as a non-smooth, noise-corrupted search problem, and rely on model-free stochastic optimizations, such as simulated annealing \cite{ciocarlie2007dexterous}, Bayesian optimization \cite{nogueira2016unscented}, and multi-armed bandits \cite{laskey2015multi}, to optimize the grasp quality. Although these methods make minimal assumptions on geometries of objects and kinematics of grippers, they are typically sample-intensive. Instead, several relatively recent works \cite{fontanals2014integrated,wang2019manipulation,Liu2020DeepDG} demonstrate the advantage of model-based approaches in terms of fast convergence \cite{Dai2018}, amenability to machine learning \cite{Liu2020DeepDG}, and global optimality \cite{Liu2020MICP,schulman2017grasping,6906885}. Model-based approaches utilize certain properties of grasp metrics, object shapes, or gripper types, such as derivatives \cite{Liu2020DeepDG}, submodularity \cite{schulman2017grasping}, and monotonicity \cite{6906885}, to guide the search of optimal grasps and achieve improved efficacy.

Despite their various advantages, model-based approaches are relatively less used due to a limited robustness and generality in several ways. Most model-based algorithms \cite{Liu2020DeepDG,Liu2020MICP,schulman2017grasping,6906885} are limited to precision grasps by pre-sampling a small set of contact points either on the gripper or the object. In comparison, model-free, sampling-based approaches are agnostic to contact points and can easily handle power grasps. Moreover, some model-based approaches \cite{schulman2017grasping,pan2020grasping} only plan grasp points without considering gripper fesasibility. Other methods \cite{6906885,Liu2020MICP,Dai2018} can account for gripper kinematics, but they either resort to model-free sampling-based method \cite{6906885}, require a long computational time \cite{Liu2020MICP}, or cannot handle complex object shapes \cite{Dai2018,Liu2020MICP}.

If we switch gears from grasp planning to general contact-rich path planning, there has been numerous efforts to sidestep the above limitations. In particular, contact-implicit trajectory optimization \cite{mordatch2012discovery,posa2014direct,manchester2019contact} generates trajectories with unprecedented complexity by allowing a numerical optimizer to make or break contact points. In this paper, we propose to borrow these techniques and design a model-based grasp planner without using pre-sampled contact points. Unlike contact-rich path planning where contacts only happen on robot end-effectors, we propose to consider every pair of points on the object and the gripper for potential contacts, and allow the optimizer to determine their status. However, there are infinitely many such point pairs, for which a na\"{\i}ve discretization is computationally intractable.

\textbf{Main Result:} We study the grasp planning problem through the lens of IPCC formulation. We introduce a pair of complementary constraints between each pair of points on the object and the gripper. Complementary constraints allow the optimizer to jointly choose contact positions, forces, and gripper's kinematic poses, during which the contact state is implicitly determined. IPCC is one of the most challenging optimization problems that are typically solved by constraint approximation or instantiation \cite{stein2012solve}. However, we show that, in the special case of grasp planning with $Q_\infty$ metric objective function, IPCC reduces to a standard NLP via the technique of kernel-integral relaxation, which reduces an infinite set of constraints to a single constraint involving an surface integral of a kernel function. Moreover, we adapt the Fast Gauss Transform (FGT) \cite{spivak2010fast}, a variant of Fast Multiple Method (FMM) \cite{beatson1997short}, to efficiently evaluate the surface integrals and its Jacobian matrix. This technique leads to significantly speed-up over brute-force evaluation, as shown in Figure \ref{fig:fmm}. Our new approach provides much larger solution space than prior works and inherently allow both precision and power grasps. Finally, we use log-barrier functions and robust line-search scheme to guarantee the satisfaction of penetration- and self-collision free constraints. We summarize the new features of our method in \prettyref{table:features}.
\vspace{-5px}
\setlength{\tabcolsep}{5pt}
\begin{table}[h]
\centering
\begin{tabular}{
>{\columncolor[gray]{0.8}}cc
>{\columncolor[gray]{0.8}}cc
>{\columncolor[gray]{0.8}}c}
\toprule
\rowcolor{gray!50}
Method & Non-Convex & Power Grasps & Collision-Free & Gripper\\
\midrule
\cite{Liu2020DeepDG}         
& $\checkmark$ & & & $\checkmark$\\
\cite{Liu2020MICP,Kragic2017}
& $\checkmark$ & & $\checkmark$ & $\checkmark$\\
\cite{schulman2017grasping,7989253}
& $\checkmark$ & & &\\
\cite{Dai2018}               
& & $\checkmark$ & & $\checkmark$\\
Ours                         
& $\checkmark$ & $\checkmark$ & $\checkmark$ & $\checkmark$\\
\bottomrule
\end{tabular}
\caption{\small{\label{table:features} We compare representative model-based grasp planners in terms of handling complex non-convex objects, planning power grasps, ensuring collision-free, planning both grasp qualities and gripper poses. Note that some methods \cite{Dai2018,Liu2020DeepDG} consider collision-free constraints but the underlying numerical model cannot ensure the constraints are satisified.}}
\vspace{-5px}
\end{table}

Our grasp planning method is fast and robust, which has been verified by batch processing 20 objects with various geometrical and topological complexities using a 3-fingered, 15-DOF Barrett Hand and a 5-fingered, 24-DOF Shadow Hand. Compared with prior state-of-the-arts, our algorithm achieves considerably less computational time than \cite{Liu2020MICP}, higher robustness to penetrations than \cite{Dai2018}, or higher quality of grasps than \cite{Liu2020DeepDG}.

\section{Related Work}
We briefly review related works in model-free and model-based grasp planning. We then provide background on contact-implicit path planning and fast multipole method.

\textbf{Model-free grasp planners} treat a grasp simulator as a black-box. All the existing model-free planners are sampling-based and inherit celebrated completeness and optimality properties \cite{vahrenkamp2010integrated,6202433}. Various techniques have been proposed to improve their efficacy. Early works \cite{eigengrasp2007,ciocarlie2007dexterous} reduce the dimension of search space by limiting the DOF of a gripper. More recent approaches utilize correlation between samples and formulate the grasp planning in Bayesian optimization \cite{nogueira2016unscented} or multi-arm bandits \cite{laskey2015multi} settings. Model-free method features a high versatility in generalizing to all kinds of 3D objects, gripper modalities, and types of grasps (see e.g. \cite{GraspIt}). These methods have recently witnessed significant progress thanks to the use of data-driven techniques, e.g. \cite{mahler2016dex,Goldberg2020}, but this topic is out of the scope of this work.

\textbf{Model-based grasp planners} exploit additional assumptions on a grasp simulator or use additional outputs from the simulator to further improve the planning performance. For example, \citewithauthor{primitive2003} assumed the 3D objects resemble some simple geometric primitives and \citewithauthor{Dai2018} assumed the 3D objects are convex. Other works make assumptions on the grasp quality metrics, \citewithauthor{6906885,Liu2020MICP} relies on the grasp metric being monotonic and \citewithauthor{schulman2017grasping} proved that $Q_{1,\infty}$ metrics are submodular and used this property to approximate optimal grasps with bounded sub-optimality. Finally, a large body of model-based planners \cite{Dai2018,wang2019manipulation,Liu2020DeepDG,Liu2020MICP} formulate the problem as gradient-based numerical optimization and require a grasp simulator to be differentiable.

\textbf{Contact-implicit optimization} \cite{tassa2012synthesis,mordatch2012discovery,posa2014direct,pan2019gpu} has proven capable of generating complex robot motion trajectories from trivial initialization. Central to these formulations is the use of position-force complementary conditions as hard constraints in a trajectory optimizer. Our method can be interpreted as a generalization of these techniques to grasp planning. The main application of contact-implicit optimization lies in legged robotics, where contacts are assumed to only happen on a few robot end-effectors. However, to enable both precision and power grasps, we need to consider all pairs of potential contact points, leading to an infinite number of decision variables. We emphasize that two prior works \cite{tassa2012synthesis,pan2019gpu} lifted the contact-on-end-effector assumption, and allows contacts to happen anywhere on the robot. However, these methods rely on smooth contact models and do not pertain (self-)collision-free guarantee.

\textbf{Fast multipole method} finds most applications in large scale numerical simulation of N-body problems using Boundary Element Methods (BEM), where each pair of two bodies have influences on each other. As a result, summing up the total influences on all bodies incur a computational cost of $\mathcal{O}(N^2)$. FMM reduces this cost to $\mathcal{O}(N\log(N))$ or even $\mathcal{O}(N)$ by aggregating bodies into clusters and approximating the cluster-wise influences using truncated Taylor or Laurent series, while the approximation error can be arbitrarily bounded (see \cite{beatson1997short} for more details). A major advantage of BEM over Finite Element Methods (FEM) \cite{Logan2000} is that BEM only uses a surface mesh while FEM requires a volume mesh. This property has been exploited in \cite{pan2020grasping} to account for object deformations under grasp. In this work, we show that infinite complementary constraints can be replaced with a single constraint involving a kernel integration, whose value and Jacobian matrix can be evaluated efficiently using the FGT \cite{spivak2010fast}.
\section{Grasp Planning as IPCC}
We first review the basics of grasp planning. We assume that there is an object with surface $S_o$ and a robot surface $S_r$ determined by the robot's configuration $\theta$, denoted as $S_r(\theta)$, both of which are 2D manifolds. A robot can apply a wrench $w(x)$ on $x\in S_o$ if and only if $x$ is in contact or $x\in S_r$. The wrench is associated with a contact force $f(x)\in \mathcal{C}(x)$ by the relationship: $w(x)=\TWO{f(x)}{x\times f(x)}^T$, where $\mathcal{C}(x)$ is the friction cone at $x$ defining feasible forces, $x\times$ is the cross-product matrix, and we assume the object's center-of-mass is placed at the origin. When the object is undergoing external wrench $w_o$, the robot must immobilize the object via an counteracting wrench $w_{sum}$ to maintain a grasp, defined as $w_{sum}\triangleq\int_{S_o}w(x) dx$. The quality of a grasp measured using $Q_\infty$ metric is defined as:
\begin{align*}
Q_\infty\triangleq
\begin{cases}
\fmin{\|w_o\|=1}\fmax{f(x)}\;&
\left<w_o,w_{sum}\right>\\
\ST\;&\left<n(x),f(x)\right>\leq1
\end{cases},
\end{align*}
where $n(x)$ is the inward normal at $x\in S_o$. Intuitively, $Q_\infty$ equals to the largest magnitude of external wrench that the robot can counteract along all possible directions, using bounded grip force. Note that the above integral must be well-defined because the constraint $\left<n(x),f(x)\right>\leq1$ makes the integrand bounded and the domain of integral is also bounded. In this paper, we consider the following discretized $Q_\infty$ by limiting $w_o$ to a finite set $w_o^1,\cdots,w_o^D$:
\begin{equation}
\begin{aligned}
\label{eq:QINF}
Q_\infty\triangleq
\begin{cases}
\fmin{d=1,\cdots,D}\;\fmax{f^d(x)}\;&
\left<w_o^d,w_{sum}\right>\\
\ST\;&\left<n(x),f^d(x)\right>\leq1
\end{cases},
\end{aligned}
\end{equation}
where $f^d$ is the contact force to resist external wrench along $w_o^d$. Combining the definition of $Q_\infty$ and the force-position complementary condition, a grasp planning problem is defined by the following IPCC:
\begin{equation}
\begin{aligned}
\label{eq:IPCC}
\argmax{\theta,f^d(x)\in C(x)}\;&Q_\infty\\
\ST\;&0\leq\left<n(x),f^d(x)\right>\perp d_r(x,\theta)\geq0,
\end{aligned}
\end{equation}
which inherently handles power grasps using infinitely many variables $f^d(x)$, each involved in a complementary constraint dictating that only points in contact can impose non-zero forces on the object. Here $d_r(x,\theta)$ is the distance between $x$ and the robot surface at configuration $\theta$.
\begin{table}
\caption{\label{table:symbols} Symbol Table.}
\resizebox{.48\textwidth}{!}{
\begin{tabular}{cc}
%a
\rowcolors{0}{gray!50}{white}
\begin{tabular}{ll}
\toprule
Variable & Definition \\
\midrule
$S_o$ & object surface\\
$S_r$ & robot surface\\
$\theta$ & robot configuration\\
$x$ & a point on object\\
$y$ & a point on robot in global coordinates\\
$R,t$ & local-to-global rotation, translation\\
$y^l$ & a point on robot in local coordinates\\
$\mathcal{C}$ & feasible force cone\\
$n(x)$ & outward normal on $x$\\
$f,w_{sum}$ & force,sum of wrench on object\\
$f^d,w_o^d$ & $d$th external force,wrench\\
$Q_\infty$ & grasp quality metric\\
$D$ & number of sampled directions\\
$d_r$ & distance to the robot\\
$\alpha$ & complementary relaxation parameter\\
$g^d(x)$ & resisting wrench on $x$\\
\hline
\end{tabular}
%b
\rowcolors{0}{gray!50}{white}
\begin{tabular}{ll}
\toprule
Variable & Definition \\
\midrule
$G^d(\theta)$ & resisting wrench for direction $d$\\
$K$ & kernel function\\
$L,L_{o,r}$ & collision avoidance term\\
$\{P,R,n,n_0\}^{pq}$ & separating plane\\
$S_r^l,V(l)$ & $l$th link, number of vertices\\
$L$ & number of links\\
$\phi,\rho$ & merit function, constraint weight\\
$\gamma$ & constraint weight in merit function\\
$r$ & radius of Poisson's disk\\
$N,M$ & number of source, target points\\
$c,b$ & center point of source, target box\\
$H_n,h_n,h_n^j$ & Hermite functions\\
$A_n,B_n^j,C_n,\{E,F,H,I\}_m$ & FGT coefficients\\
$n_0$ & number of truncated terms in FGT\\
$\mathcal{B}$ & clustering box of FGT\\
$\mathcal{S}(y)$ & source strength\\
\hline
\end{tabular}
\end{tabular}}
\vspace{-10px}
\end{table}
\section{\label{sec:IPCCReduction}Kernel-Integral Reduction}
In this section, we propose a practical reformulation of \prettyref{eq:IPCC} as a standard NLP by using the relaxed complementary constraint \cite{hoheisel2013theoretical}. Each complementary constraint is equivalent to three inequalities:
\begin{align*}
\begin{cases}
&\left<n(x),f^d(x)\right>\geq0\\
&d_r(x,\theta)\geq0\\
&\left<n(x),f^d(x)\right>d_r(x,\theta)\leq0,
\end{cases}
\end{align*}
and \citewithauthor{hoheisel2013theoretical} proposed to replace the third inequality with $\left<n(x),f^d(x)\right>d_r(x,\theta)\leq\alpha$ for some small, positive relaxation constant $\alpha$, and then use Sequential Quadratic Programming (SQP) to satisfy a sequence of relaxed, differentiable constraints with a monotonically decreasing series of $\alpha$ that tends to zero. However, SQP cannot handle our relaxed form due to non-differentiable term $d_r$, the distance between a point and a general surface of the robot. To sidestep this incompatibility, we rewrite $d_r(x,\theta)=\fmin{y\in S_r(\theta)}\|x-y\|$ and replace each relaxed complementary constraint with an infinite set:
\begin{align*}
\left<n(x),f^d(x)\right>\|x-y\|\leq\alpha\quad\forall y\in S_r(\theta).
\end{align*}
With a slight rearrangement and by introducing a so-called kernel function $K(\bullet,\alpha)\triangleq \alpha/\bullet$, each complementary constraint takes the form: 
\begin{equation}
\begin{aligned}
\label{eq:KKT1}
&\left<n(x),f^d(x)\right>\leq K(\|x-y\|,\alpha)\quad\forall y\in S_r(\theta)\\
&0\leq\left<n(x),f^d(x)\right>\leq 1,
\end{aligned}
\end{equation}
where we have merged the requirement of $Q_\infty$ that normal force magnituide is less than $1$. We show that, as $\alpha\to0$, the infinite set of constraint \prettyref{eq:KKT1} is equivalent to the following single constraint for a specific choice of kernel function $K(\bullet,\alpha)$:
\begin{align}
\label{eq:KKT2}
\left<n(x),f^d(x)\right>\leq\int_{S_r}K(\|x-y\|,\alpha)dy.
\end{align}
\begin{lemma}
\label{lem:KERNEL}
Suppose we choose:
\begin{align*}
K(\bullet,\alpha)=\frac{-1}{(2\pi\log\alpha)(\bullet^2+\alpha^2)},
\end{align*}
and the constraint \prettyref{eq:KKT2} is satisfied for a monotonic sequence $0\leq\alpha^k\to0$:
\begin{align*}
0\leq\left<n(x^k),f^d(x^k)\right>\leq\int_{S_r}K(\|x^k-y\|,\alpha^k)dy,
\end{align*}
and there is a convergence subsequence that tends to $x^*$, then we have $\left<n(x^*),f^d(x^*)\right>d_r(x^*,\theta)\leq0$ and $\left<n(x^*),f^d(x^*)\right>\leq1$.
\end{lemma}
\begin{proof}
Without loss of generality, we can assume the entire sequence $x^k$ is convergent to $x^*$. \TE{Case I:} If $d_r(x^*,\theta)>0$, then by the choice of kernel function we have $K(\|x^k-y\|,\alpha^k)\to0$ and \prettyref{eq:KKT2} implies $\left<n(x^*),f^d(x^*)\right>=0$. \TE{Case II:} If $d_r(x^*,\theta)=0$, then there is a unique point $y^*\in S_r$ such that $\|x^*-y^*\|=0$ and the integral is singular at $y^*$. The integral is thereby nonzero only within an infinitesimal disk around $y^*$ with radius $\delta r$. By changing the integral under polar coordinates, we have:
\small
\begin{align*}
\lim_{\alpha\to0}\int_0^{\delta r} \frac{-r}{\log\alpha(r^2+\alpha^2)} dr
=\lim_{\alpha\to0}\frac{\log\alpha-\frac{1}{2}\log(\delta r^2+\alpha^2)}{\log\alpha}=1.
\end{align*}
\normalsize
We conclude that \prettyref{eq:KKT2} is an appropriate equivalence of \prettyref{eq:KKT1} in the limit of $\alpha$.
\end{proof}
Note that $K$ does not need to take the exact form as in \prettyref{lem:KERNEL} in practice. This is because \prettyref{lem:KERNEL} only considers the limiting behavior of $K$ when $\alpha\to0$, but we would terminate optimization with a finite, positive $\alpha$ due to limited machine precision. Our experiments show that it suffice to choose any $K$ that decay quickly as $r\to\infty$. Indeed, we find that choosing $K$ to be an exponential function would lead to an efficient algorithm for evaluating the integral in \prettyref{eq:KKT2} and refer readers to \prettyref{sec:integral} for more details.
 
Next, we show that $f^d$ has closed-form solution. We notice the inner max function in \prettyref{eq:QINF} can be moved into the integral, giving:
\begin{align*}
\fmin{d=1,\cdots,D}\int_{S_o}\fmax{f^d(x)}\left<w_o^d,w^d(x)\right>dx,
\end{align*}
where the integrand is the only term related to $w^d(x)$ and $w^d(x)$ is positively proportional to $Q_\infty$. If we fix all other variables, $f^d(x)$ is the solution of the following subproblem:
\begin{equation}
\begin{aligned}
\label{eq:FD}
\argmax{f^d(x)\in\mathcal{C}(x)}\;&\left<w_o^d,w^d(x)\right>\\
\ST\;&0\leq\left<n(x),f^d(x)\right>\leq\int_{S_r}K(\|x-y\|,\alpha)dy.
\end{aligned}
\end{equation}
Using a similar reasoning as \cite{Liu2020DeepDG,schulman2017grasping}, the solution to \prettyref{eq:FD} has a closed form:
\begin{align*}
&f^d(x)=g^d(x)\int_{S_r}K(\|x-y\|,\alpha)dy\\
&g^d(x)\triangleq
\begin{cases}
\argmax{f^d(x)\in\mathcal{C}(x)}\;&\left<w_o^d,w^d(x)\right>\\
\ST\;&0\leq\left<n(x),f^d(x)\right>\leq1
\end{cases},
\end{align*}
and we refer readers to \cite{schulman2017grasping} for the derivation of the expression of $g^d(x)$. When plugged into \prettyref{eq:IPCC}, the closed-form solution already incorporates the relaxed complementary constraints and eliminates all the complementary variables, thus reducing the IPCC to the following standard NLP:
\begin{equation}
\begin{aligned}
\label{eq:NLP}
\argmax{\theta}\;&Q_\infty\triangleq\fmin{d=1,\cdots,D}\;G^d(\theta)\\
&G^d(\theta)\triangleq\int_{S_o}g^d(x)dx\int_{S_r}K(\|x-y\|,\alpha)dy\\
\ST\;&d_r(x,\theta)\geq0,
\end{aligned}
\end{equation}
which provides a variational explanation of $Q_\infty$ that allows any point on the robot surface to make contact with any other point on the object, thereby unifying precision and power grasps. The choice of grasp points is implicitly encoded in the double integral over the object and robot surfaces. We will show that such integrals can be approximately efficiently using FMM.

\subsection{Guaranteed (Self-)Collision-Free}
\prettyref{eq:NLP} is still semi-infinite due to the infinitely many collision constraints: $d_r(x,\theta)\geq0$. In prior work \cite{Liu2020DeepDG}, the collision-free constraint $d_r\geq0$ is imposed using soft penalty terms, which is not guaranteed to be satisfied. We propose to ensure collision-free via the log-barrier function:
\begin{align}
\label{eq:collision}
L_o(\theta)=-\int_{S_o}\log\left[d_r(x,\theta)\right]dx.
\end{align}
Using a line-search algorithm, we can guarantee that $L_o$ takes a finite value throughout the optimization, which in turn implies collision-free between the robot and the gripper. In practice, we assume the object is provided as a point cloud and replace the integral of $S_o$ with a summation over each point. $L_o$ is differentiable as shown in \cite{Liu2020DeepDG} and the evaluation of summation can be accelerated using a bounding volume hierarchy and log-barrier function with local support (see \cite{ruiqi2021} for more details).

We further consider self-collision assuming each robot link takes a convex shape. Assuming that the robot surface is decomposed into $L$ links $S_r=\bigcup_{l=1}^LS_r^l$ where each $S_r^l$ is the convex hull of $V(l)$ vertices $\{y_1^l(\theta),\cdots,y_{V(l)}^l(\theta)\}$. Then a separating plane $P^{pq}(y)=\left<n^{pq},y\right>+n_0^{pq}$ could be introduced to avoid collision between a pair of links $S_r^{p,q}$, where $n^{pq},n_0^{pq}$ are plane normal and offset. The log-barrier function for self-collision takes the following form:
\begin{align*}
L_r(\theta,n^{pq},n_0^{pq})=&-\sum_{1\leq p<q\leq L}
\sum_{i=1}^{V(p)}\log\left[\left<n^{pq},y_i^p(\theta)\right>+n_0^{pq}\right]\\
&-\sum_{1\leq p<q\leq L}
\sum_{j=1}^{V(q)}\log\left[-\left<n^{pq},y_j^q(\theta)\right>-n_0^{pq}\right].
\end{align*}
We propose to use block coordinate descend algorithm and interleave the optimization for $\theta$ and $P^{pq}$, so that the optimization for each plane $P^{pq}$ is independent. To ensure the plane normal has unit length, we use reparameterize $n^{pq}=R^{pq}e$ with $R^{pq}\in\mathcal{SO}(3)$ represented using Rodriguez formula and $e$ being an arbitrary unit vector.

\subsection{Simplified SQP for Minimizing $Q_\infty$}
Putting everything together, we recast NLP (\prettyref{eq:NLP}) as an unconstrained optimization:
\begin{align}
\label{eq:opt}
\argmin{\theta,P^{pq}}\;L(\theta)-Q_\infty(\theta)\quad L\triangleq L_o+L_r,
\end{align}
which can be solved using a simplified SQP algorithm. The non-differentiable $\min$ operator in $Q_\infty$ can be replace with hard constraints:
\begin{align}
\label{eq:optCons}
\argmin{\theta,P^{pq},Q}\;L(\theta)-Q\quad\ST\;Q\leq G_d(\theta),
\end{align}
where $Q$ is a slack variable. We show in our appendix that SQP takes a simplified form when solving \prettyref{eq:optCons} by observing that the QP subproblem is always feasible.
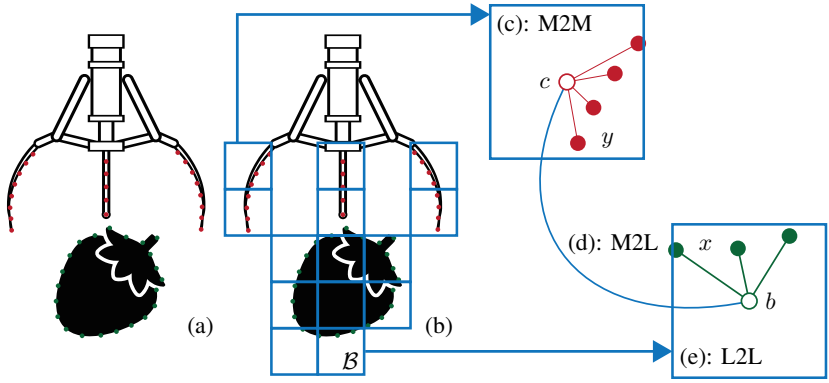
\begin{figure*}
\begin{minipage}[b]{0.7\linewidth}
\centering
\scalebox{0.9}{
\includegraphics[width=0.98\linewidth]{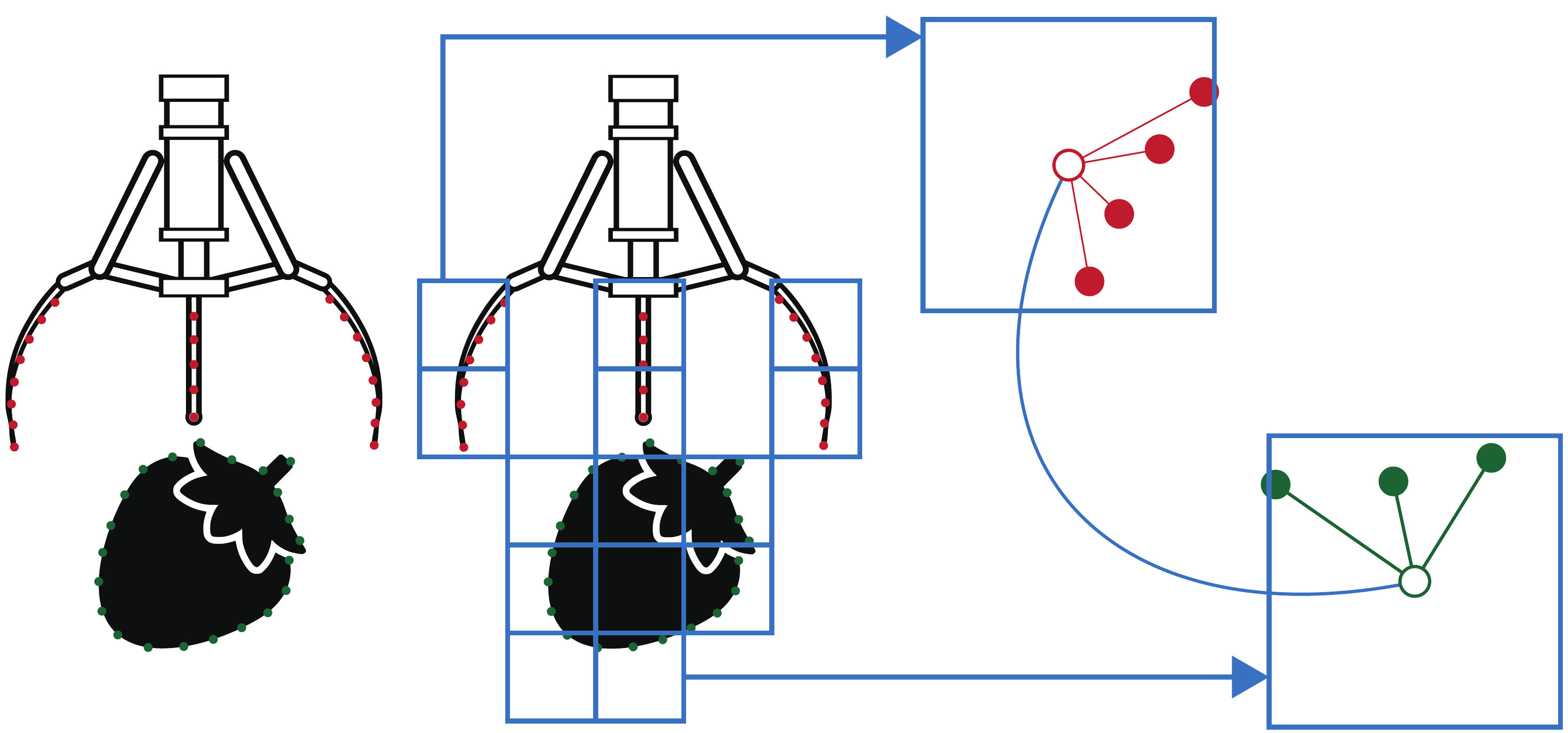}
\put(-270,20){(a)}
\put(-170,20){(b)}
\put(-140,147){(c): M2M}
\put(-110,55){(d): M2L}
\put(-63,7){(e): L2L}
\put(-205,5){$\mathcal{B}$}
\put(-96,100){$y$}
\put(-122,123){$c$}
\put(-27,30){$b$}
\put(-55,55){$x$}}
\vspace{-10px}
\end{minipage}
\hfill
\begin{minipage}[b]{.29\textwidth}
\captionof{figure}{\small{We illustrate FGT applied to grasp planning. (a): We sample possible contact points both on the gripper (red) and the object surface (green). (b): The number of sample points is large and we cluster them into axis-aligned boxes $\mathcal{B}$ (blue). FGT works in three steps. (c): M2M step substitutes the contributions (to $G^d$) of source point $y$ with the center point $c$ using Hermite expansion (red line). (d): M2L step substitutes the contributions of center point $c$ with the center point $b$ using Taylor expansion (blue line). (e): L2L step evaluates $G^d$ for target point $x$ around $b$ (green line).}}
\label{fig:FGT}
\vspace{-20px}
\end{minipage}
\end{figure*}
\section{\label{sec:integral}Numerical Integral Evaluation}
Although we have derived the standard NLP \prettyref{eq:optCons}, the integrals involved in $G_d$ and $L(\theta)$ do not have analytic expressions and need to be evaluated numerically. The double integral involved in $G_d$ is known as Fredholm integral of the first kind, where the integrand is a multiplication of a source term $g^d(x)$ and a kernel function $K(\|x-y(\theta)\|,\alpha)$ that is singular when $x$ is close to $y(\theta)$ and decay quickly as the distance increases. An intuitive method to discretize $G^d$ would sample the two surfaces $S_o$, $S_r$ with dense set of $N$ points $\{x\}\in S_o$ and $M$ points $\{y(\theta)\}\in S_r(\theta)$ using Poisson disk sampling with radius $r$ and approximate $G^d$ with double integral:
\begin{align*}
G^d\approx(\pi r^2)^2
\sum_{x}g^d(x)\sum_{y}K(\|x-y\|,\alpha),
\end{align*} 
which incurs a cost of $\mathcal{O}(NM)$. We introduce a modified FGT \cite{spivak2010fast}, a variant of FMM that can be applied if $K$ is chosen to be $K(\bullet,\alpha)=\exp(-\bullet^2/\alpha)$. The standard FGT would only computes $G^d$ and we derive extra equations to evaluate $\FPPR{G^d}{\theta}$ as required by SQP. (We use the same set of samples to discretize the integral in \prettyref{eq:collision}, the computational cost of which is $\mathcal{O}(N)$. Compared with $G^d$, the cost to evaluate \prettyref{eq:collision} is marginal.)

We use $x$ to denote a point on the object and $y$ denotes a point on the gripper. As illustrated in \prettyref{fig:FGT}, FGT first cluster all the sampled points into boxes of side length $2\sqrt{\alpha}$, where each of $\mathcal{B}_x$ and $\mathcal{B}_y$ denotes a box that contains some point $x,y$, respectively. For each $\mathcal{B}_y$, FGT first uses Multipole-to-Multipole (M2M) step to approximate their contribution (to the integral) via Hermite expansion. Then for each $\mathcal{B}_{x,y}$ pair, FGT uses Multipole-to-Local (M2L) step to transfer the contribution from $\mathcal{B}_y$ to $\mathcal{B}_x$. Finally, FGT uses Local-to-Local (L2L) step to distribute the contribution from the $\mathcal{B}_x$ to each $x$.

\begin{algorithm}
\caption{\label{alg:FGT}FGT}
\begin{algorithmic}[1]
\Require{Initial error threshold $\epsilon$ and $n_0(\epsilon)$}
\State Cluster all $x,y$ into boxes of side length $2\sqrt{\alpha}$
\For{$\mathcal{B}_y$}
\State $c\gets$ center of $\mathcal{B}_y$
\State Precompute $A_n,B_n^j,C_n$ (\prettyref{eq:M2M},\ref{eq:M2MG})
\EndFor
\For{$\mathcal{B}_{x,y}$ pair}
\State $c\gets$ center of $\mathcal{B}_y$
\State $b\gets$ center of $\mathcal{B}_x$
\State Precompute $E_m,F_m,H_m,I_m$ (\prettyref{eq:M2L},\ref{eq:M2LG})
\EndFor
\For{$\mathcal{B}_x$}
\State Compute $G^d,\FPP{G^d}{\theta}$ (\prettyref{eq:L2LG})
\EndFor
\end{algorithmic}
\end{algorithm}

\subsection{M2M Step}
Assuming $x,y$ are two 1D points, the FGT is based on the Hermite expansion of exponential function:
%\small
\begin{align*}
\ifdetail{
&\exp(2yx-x^2)=\sum_{|n|=0}^\infty\frac{x^n}{n!}H_n(y)\\
\Rightarrow&\exp(-(y-x)^2)=
\sum_{|n|=0}^\infty\frac{x^n}{n!}H_n(y)\exp(-y^2)
\triangleq\sum_{|n|=0}^\infty\frac{x^n}{n!}h_n(y)\\
\Rightarrow&
}\fi
\exp(-(\frac{y-x}{\sqrt{\alpha}})^2)=
\sum_{|n|=0}^\infty\frac{1}{n!}
(\frac{c-y}{\sqrt{\alpha}})^n
h_n(\frac{x-c}{\sqrt{\alpha}}),
\end{align*}
%\normalsize
where $H_n$ are Hermite polynomials and $c$ is the center point of $\mathcal{B}_y$. If $x,y,c$ are 3D points, then we use subscript to denote the coordinate index and the expansion takes the same form as above but $n$ is a vector $\THREE{n_1}{n_2}{n_3}$. We have $n!\triangleq n_1!n_2!n_3!$, $|n|\triangleq n_1+n_2+n_3$, $r^n\triangleq\Pi_{i=1}^3r_i^{n_i}$, and $h_n(r)\triangleq\Pi_{i=1}^3h_{n_i}(r_i)$. The gradient with respect to $y_i$ has the following Hermite expansion:
%\small
\begin{align*}
&\FPP{}{y_j}\left[\exp(-(\frac{y-x}{\sqrt{\alpha}})^2)\right]=
\frac{2(x_j-y_j)}{\alpha}\exp(-(\frac{y-x}{\sqrt{\alpha}})^2)\\
=&\sum_{|n|=0}^\infty\frac{-2}{\sqrt{\alpha}n!}
\left[(\frac{c-y}{\sqrt{\alpha}})^{n+e_j}
h_n(\frac{x-c}{\sqrt{\alpha}})+
(\frac{c-y}{\sqrt{\alpha}})^n
h_n^j(\frac{x-c}{\sqrt{\alpha}})\right],
\end{align*}
%\normalsize
where $h_n^j(r)\triangleq r_jh_n(r)$. The two above expansions form the Multipole-to-Multipole (M2M) step of FGT. If there is a set of points $y\in\mathcal{B}_y$ around a center point $c$, then we have:
%\small
\begin{equation}
\begin{aligned}
\label{eq:M2M}
&\sum_{y\in\mathcal{B}_y}\mathcal{S}(y)\exp(-(\frac{y-x}{\sqrt{\alpha}})^2)
=\sum_{|n|=0}^\infty A_n h_n(\frac{x-c}{\sqrt{\alpha}})\\
&A_n\triangleq\sum_{y\in\mathcal{B}_y}\mathcal{S}(y)\frac{1}{n!}(\frac{c-y}{\sqrt{\alpha}})^n.
\end{aligned}
\end{equation}
%\normalsize
Similarly for the gradient, we have:
%\small
\begin{equation}
\begin{aligned}
\label{eq:M2MG}
&\sum_{y\in\mathcal{B}_y}\mathcal{S}(y)\FPP{}{y_j}
\left[\exp(-(\frac{y-x}{\sqrt{\alpha}})^2)\right]\\
=&\sum_{|n|=0}^\infty B_n^j h_n(\frac{x-c}{\sqrt{\alpha}})+
\sum_{|n|=0}^\infty C_n h_n^j(\frac{x-c}{\sqrt{\alpha}})\\
&B_n^j\triangleq\sum_{y\in\mathcal{B}_y}
\mathcal{S}(y)\frac{-2}{\sqrt{\alpha}n!}(\frac{c-y}{\sqrt{\alpha}})^{n+e^j}
\quad C_n=\frac{-2}{\sqrt{\alpha}}A_n,
\end{aligned}
\end{equation}
%\normalsize
where $\mathcal{S}(y)$ is some $y$-dependent coefficients. The M2M step involves dividing the space into a set of axis-aligned boxes $\mathcal{B}_y$ with side length $2\sqrt{\alpha}$. For all the source points $y$ belonging to a $\mathcal{B}_y$, M2M identifies their contributions with a single center point $c$ using Hermite expansion (\prettyref{eq:M2M} and \prettyref{eq:M2MG}). FGT only retains terms with $n\leq n_0$, where $n_0$ is chosen to ensure error is small than a user chosen threshold (see \cite{spivak2010fast} for more details).

\subsection{M2L Step}
The center points $c$ can still be faraway from target points $x$. M2L step identifies the contributions of center points $c$ with some other points $b$ that is close to target points using Taylor expansion. A Hermite expansion has the following equivalent Taylor expansion:
\small
\begin{equation}
\begin{aligned}
\label{eq:M2L}
&\sum_{|n|=0}^\infty A_n h_n(\frac{x-c}{\sqrt{\alpha}})
=\sum_{|m|=0}^\infty E_m (\frac{x-b}{\sqrt{\alpha}})^m\\
&E_m\triangleq\frac{(-1)^{|m|}}{m!}\sum_{|n|=0}^\infty A_nh_{n+m}(\frac{c-b}{\sqrt{\alpha}}).
\end{aligned}
\end{equation}
\normalsize
For the gradient, we have:
\small
\begin{equation}
\begin{aligned}
\label{eq:M2LG}
\ifdetail{
&\sum_{|n|=0}^\infty B_n^j h_n(\frac{x-c}{\sqrt{\alpha}})=
\sum_{|m|=0}^\infty F_m (\frac{x-b}{\sqrt{\alpha}})^m\\
&\sum_{|n|=0}^\infty C_n h_n^j(\frac{x-c}{\sqrt{\alpha}})=
\frac{x_j-c_j}{\sqrt{\alpha}}\sum_{|n|=0}^\infty C_n h_n(\frac{x-c}{\sqrt{\alpha}})\\
=&\frac{x_j-b_j}{\sqrt{\alpha}}\sum_{|n|=0}^\infty C_n h_n(\frac{x-c}{\sqrt{\alpha}})+
\frac{b_j-c_j}{\sqrt{\alpha}}\sum_{|n|=0}^\infty C_n h_n(\frac{x-c}{\sqrt{\alpha}})\\
=&\sum_{|m|=0}^\infty \left[
H_m (\frac{x-b}{\sqrt{\alpha}})^{m+e_j}+
I_m (\frac{x-b}{\sqrt{\alpha}})^m\right]\\
}\fi
&\sum_{y\in\mathcal{B}_y}\mathcal{S}(y)\FPP{}{y_j}
\left[\exp(-(\frac{y-x}{\sqrt{\alpha}})^2)\right]\\
=&\sum_{|m|=0}^\infty \left[
H_m (\frac{x-b}{\sqrt{\alpha}})^{m+e_j}+
(F_m+I_m) (\frac{x-b}{\sqrt{\alpha}})^m\right]\\
&F_m\triangleq\frac{(-1)^{|m|}}{m!}\sum_{|n|=0}^\infty B_n^jh_{n+m}(\frac{c-b}{\sqrt{\alpha}})\\
&H_m\triangleq\frac{(-1)^{|m|}}{m!}\sum_{|n|=0}^\infty C_nh_{n+m}(\frac{c-b}{\sqrt{\alpha}})\\
&I_m\triangleq\frac{(-1)^{|m|+1}}{m!}\sum_{|n|=0}^\infty C_nh_{n+m}^j(\frac{c-b}{\sqrt{\alpha}}).
\end{aligned}
\end{equation}
\normalsize
Again we only retain all the terms with $m\leq n_0$. The M2L step involves dividing the space into another set of axis-aligned boxes $\mathcal{B}_x$ with side length $2\sqrt{\alpha}$. For each pair of boxes with center points $c,b$, M2L transfers the contribution from $c$ to $b$ (\prettyref{eq:M2L} and \prettyref{eq:M2LG}). This only needs to be done for pairs of boxes that are certain distances away.

\begin{figure*}[t]
\centering
\begin{tabular}{cc}
\includegraphics[width=\columnwidth]{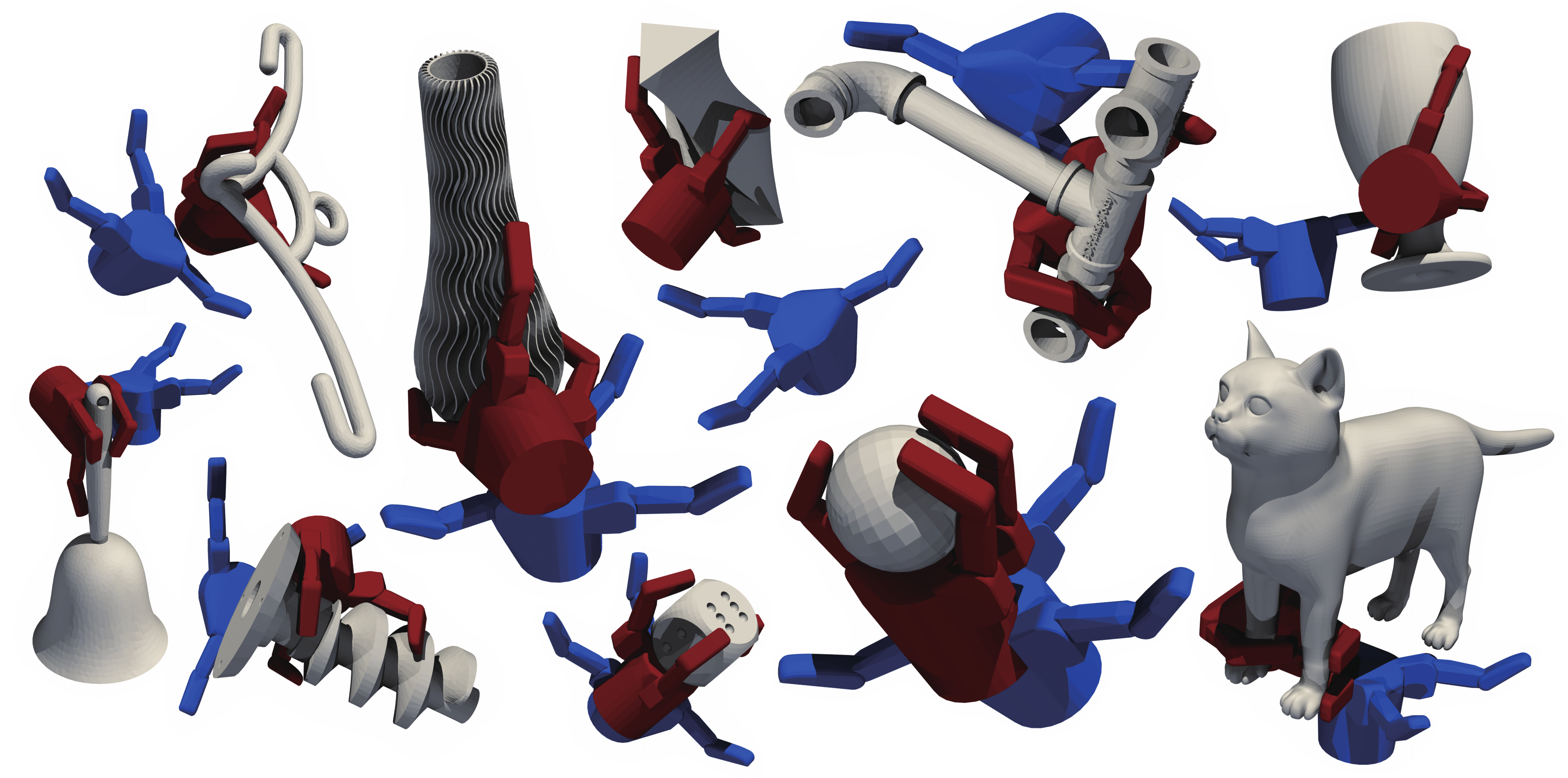}
\includegraphics[width=\columnwidth]{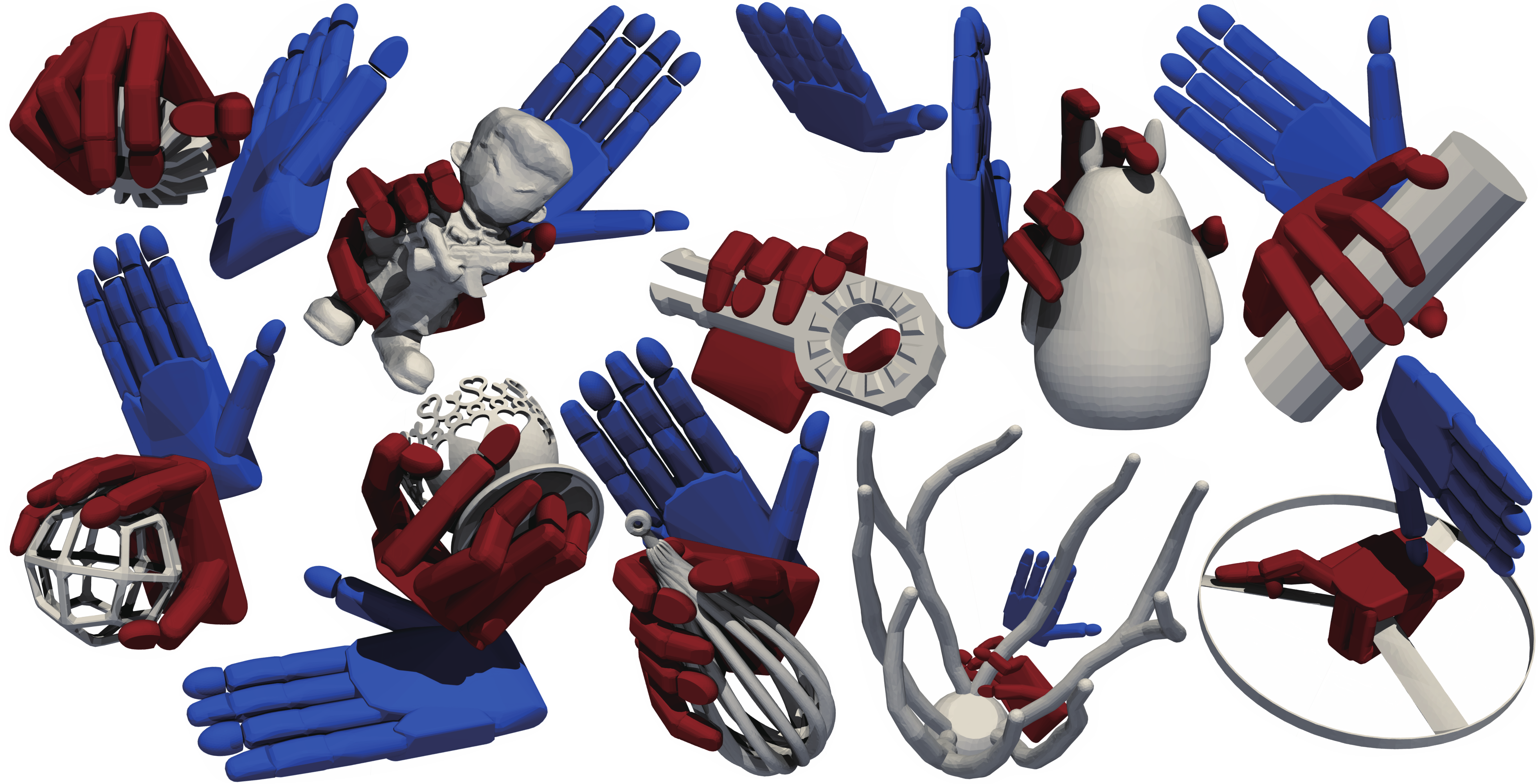}
\end{tabular}
\caption{\label{fig:data}We apply our method to grasp 10 complex objects using Barrett Hand (left) and Shadow Hand (right).}
\end{figure*}

\setlength{\tabcolsep}{1pt}
\begin{table*}[tbp]
\centering
\begin{tabular}
{l>{\columncolor[gray]{0.8}}
cc>{\columncolor[gray]{0.8}}
cc>{\columncolor[gray]{0.8}}
cc>{\columncolor[gray]{0.8}}
cc>{\columncolor[gray]{0.8}}
cc}
\toprule
\rowcolor{gray!50}
Barrett Hand
%\diagbox[width=10em]{Algorithm}{Model} 
& 1 & 2 & 3 & 4 & 5 & 6 & 7 & 8 & 9 & 10  \\
\midrule
Ours                        & 
\convert{1.89942e-06}       & 
\convert{2.29199e-06}       & 
\convert{1.25874e-05}       & 
\convert{2.42596e-06}       & 
\convert{7.04648e-06}       & 
\convert{2.31098e-05}       & 
\convert{2.17305e-05}       & 
\convert{5.29022e-06}       & 
\convert{2.1363e-06}        & 
\convert{1.68546e-05}       \\
$Q_1$-\cite{Liu2020DeepDG}  & 
\convert{2.14656e-07}       & 
\convert{1.33578e-14}       & 
\convert{2.45433e-06}       & 
\convert{1.78334e-06}       & 
\convert{1.50974e-06}       & 
\convert{1.40752e-05}       & 
\convert{7.95582e-06}       & 
\convert{1.87743e-06}       & 
\convert{1.97438e-06}       & 
\convert{3.22105e-06}       \\
Closeness                   &
\convert{1.90409e-06}       &
\convert{6.75748e-07}       &
\convert{1.04532e-05}       &
\convert{2.22011e-06}       &
\convert{7.61138e-06}       &
\convert{1.63571e-05}       &
\convert{2.11329e-05}       &
\convert{1.67361e-06}       &
\convert{2.95927e-06}       &
\convert{3.42978e-06}       \\
% Closeness-\cite{GraspIt} & & & & & & & & & &  \\
$Q_1$-\cite{GraspIt}        &
\convert{1.61077e-06}       &
\convert{2.18296e-06}       &
\convert{9.39794e-08}       &
\convert{3.12716e-06}       &	
\convert{3.37161e-06}       &
\convert{1.4781e-05}        &
\convert{1.94229e-07}       &
\convert{1.10931e-08}       &
\convert{2.38429e-06}       &
\convert{5.11875e-09}       \\
\bottomrule
\toprule
\rowcolor{gray!50}
Shadow Hand
%\diagbox[width=10em]{Algorithm}{Model} 
& 1 & 2 & 3 & 4 & 5 & 6 & 7 & 8 & 9 & 10  \\
\midrule
Ours                        & 
\convert{5.32004e-06}       &
\convert{6.46704e-06}       &
\convert{1.16872e-05}       &
\convert{7.71481e-06}       &
\convert{1.4801e-05}        &
\convert{1.77503e-05}       &
\convert{1.15681e-05}       &
\convert{1.14212e-05}       &
\convert{1.01062e-05}       &
\convert{8.95518e-06}       \\
$Q_1$-\cite{Liu2020DeepDG}  & 
\convert{3.76495e-06}       &	
\convert{1.13016e-06}       &	
\convert{4.94282e-19}       &	
\convert{2.19937e-06}       & 	
\convert{2.6131e-06}        &	
\convert{4.339e-06}         &  
\convert{8.10799e-06}       &	
\convert{9.66135e-06}       &	
\convert{7.04694e-06}       &	
\convert{2.8047e-06}	    \\
Closeness                   &
\convert{3.997e-06}         &
\convert{7.40967e-06}       &
\convert{2.93737e-06}       &
\convert{4.4322e-06}        &
\convert{6.34187e-06}       &
\convert{ 8.29841e-06}      &
\convert{5.1788e-06}        &
\convert{5.76866e-06}       &
\convert{3.54187e-06}       &
\convert{2.70617e-06}       \\
% Closeness-\cite{GraspIt} & & & & & & & & & &  \\
$Q_1$-\cite{GraspIt}&
\convert{7.94267e-07}&
\convert{3.99198e-06}&
\convert{8.02763e-06}&
\convert{4.90842e-06}&
\convert{4.30632e-07}&
\convert{2.99829e-06}&
\convert{4.15036e-06}&
\convert{7.45751e-06}&
\convert{7.61735e-06}&
\convert{2.67667e-07}\\
\bottomrule
\end{tabular}
\caption{\small{\label{table:quality} A comparison of grasp quality ($Q_\infty$) using different algorithms on Barrett Hand (top row) and Shadow Hand (bottom row). From top to bottom: our method, differentiable grasp planner \cite{Liu2020DeepDG} guided by sub-gradients, our method with objective replaced by closeness measure, and EigenGrasp \cite{GraspIt} using $Q_1$ objective function.}}
\vspace{-5px}
\end{table*}
\subsection{L2L Step}
After substituting the center of expansion from $b$ to $c$, L2L step evaluates $G^d$ around some target point $x$ contained in a box, $\mathcal{B}_x$, with center point $b$ using (\prettyref{eq:M2L} and \prettyref{eq:M2LG}). In summary, the cost of evaluating each $G^d$ is $O(N+M)$ by setting $\mathcal{S}(y)=(\pi r^2)^2$. To evaluate $\FPPR{G^d}{\theta}$, we assume that the rigid object is an articulated body so that $y(\theta)=R(\theta)y^l+t(\theta)$ where $R(\theta),t(\theta)$ are the rotation and translation of a rigid link, and $y^l$ is the point $y$ in local coordinates of the robot link. By the chain rule, we have:
\small
\begin{equation}
\label{eq:L2LG}
\begin{aligned}
&\FPP{G^d}{\theta}=\FPP{G^d}{\TWO{R}{t}}\FPP{\TWO{R}{t}}{\theta}\\
&\FPP{G^d}{t_j}=\FPP{G^d}{y_j}\quad\FPP{G^d}{R_{ij}}=\FPP{G^d}{y_i}y_j^l.
\end{aligned}
\end{equation}
\normalsize
We first evaluate $\FPP{G^d}{\TWO{R}{t}}$ and then multiple by $\FPP{\TWO{R}{t}}{\theta}$. Each evaluation of $\FPP{G^d}{t_j}$ can be performed using FGT by setting $\mathcal{S}(y)=(\pi r^2)^2$, and each evaluation of $\FPP{G^d}{R_{ij}}$ can be performed by setting $\mathcal{S}(y)=(\pi r^2)^2 y_j^l$. Using the articulated body algorithm \cite{featherstone2014rigid}, the multiplication by $\FPP{\TWO{R}{t}}{\theta}$ incurs $\mathcal{O}(|\theta|)$. Altogether, the cost of evaluating $G^d,\FPP{G^d}{\theta}$ is $\mathcal{O}(13(N+M)+|\theta|)$ and the cost of evaluating all the constraint gradients is $\mathcal{O}((13(N+M)+|\theta|)D)$. We further notice that M2M and M2L steps are irrelevant to the $D$ wrench directions and need to be done only once, so the ultimate cost is: $\mathcal{O}(13(N+MD)+|\theta|D)$. We summarize FGT in \prettyref{alg:FGT}.
\section{Results}
To validate the effectiveness of our approach, we employ a small dataset (\prettyref{fig:data}) containing 20 models from the Thingi10k object dataset \cite{thingy}, which is divided into two groups. The first group of 10 objects are to be grasped using the (6+4)-DOF three-fingered Barrett Hand \cite{barretthand} and the second group is to be grasped using the (6+22)-DOF Shadow Hand \cite{shadowhand}. All experiments are carried out on a machine with 2.3 GHz 8-Core Intel Core i9 CPU. For all the experiments, we choose $D=128, \alpha=10^{-3}, \gamma=0.1, \beta=0.5, c=0.1, \tau=10^{-10}$. We choose $n_0$ to ensure FMM approximation error is less than $10^{-6}$ according to \cite{spivak2010fast}.

\textbf{Robustness:} Our algorithm successfully processed the entire dataset, where the objects exhibit high geometrical and topological complexities including both thin and tiny features that are oftentimes challenging in terms of collision-avoidance and contact point selection. However, our method can find human-like solutions (red poses in \prettyref{fig:data}) from trivial initializations (blue poses in \prettyref{fig:data}). The grasp quality optimized using different algorithms are summarized in \prettyref{table:quality} (Larger numbers in \prettyref{table:quality} indicate better quality and all the numbers have small absolute values due to scaling of objects). As compared with GraspIt \cite{GraspIt}, our method achieves Min/Average/Max $Q_{\infty}$ improvement rate of $0.78/402.30/3292.72$ on the BarrettHand and $1.33/9.07/34.37$ on the ShadowHand. This is the first time for model-based, optimization-based grasp planners to generate results of this level of complexity.

\textbf{Comparisons:} We have also compared our method with two prior gradient-based grasp planner. The first method is our prior work \cite{Liu2020DeepDG}, where we use sub-gradients of the $Q_1$ metric to optimize grasp poses. The second method uses the closeness energy as objective function, which minimizes the distance between point on grippers and object surfaces. The closeness energy has also been used by \cite{GraspIt}. Note that we compare all these methods in terms of the $Q_\infty$ metric for fairness. According to \prettyref{table:quality}, our method significantly outperforms both these methods. We found that the method in \cite{Liu2020DeepDG} requires a near-optimal initial guess and they rely on groundtruth data to derive initial guesses. By starting from trivial initial guesses as in \prettyref{fig:data}, sub-gradients cannot find meaningful grasps. On the other hand, the closeness energy does not consider force equilibrium condition.

\textbf{FMM Acceleration:} In \prettyref{fig:fmm} we plot the averaged iteration cost of SQP, with and without FMM acceleration. The accelerated SQP solver achieves up to $5.6\times$ speedup as compared with brute-force summation at the highest density of sampled contact points. The use of FMM never deteriorate the quality of planned grasps, achieving almost identical results as compared with brute-force summation as illustrated in \prettyref{fig:fmmQuality}. For reference, we plot the $Q_\infty$ metric computed via brute-force summation for small densities, because the cost for larger densities. We also observe improved optimized qualities when using a higher density, which will ultimately converge.
\begin{figure}
\centering
\vspace{-5px}
\includegraphics[width=\linewidth]{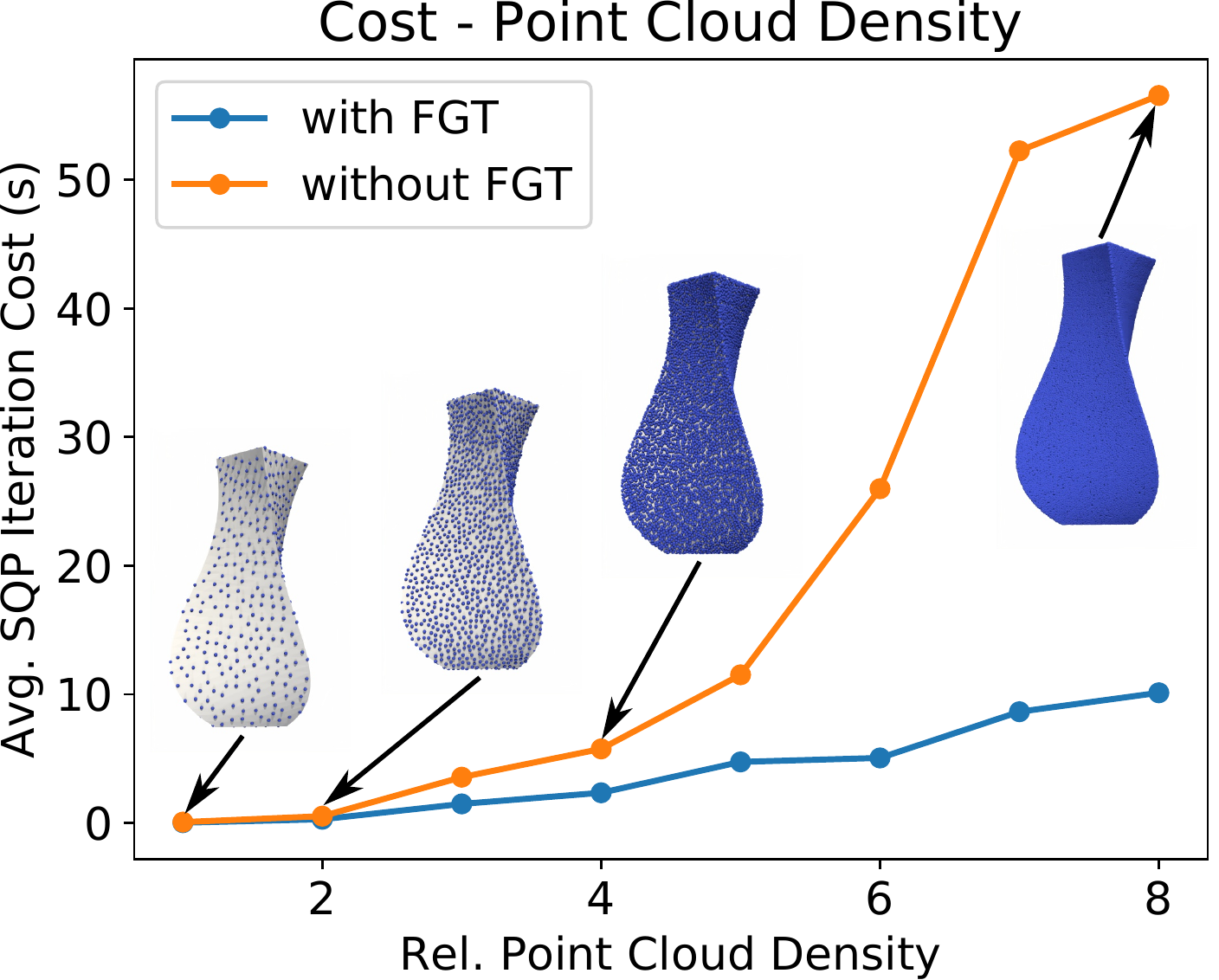}
\caption{\label{fig:fmm} The averaged computational cost of each SQP iteration plotted against the relative sample density. FMM acceleration achieves up to $5.6\times$ speedup over brute-force summation at the highest sample density.}
\vspace{-5px}
\end{figure}
\begin{figure}
\centering
\includegraphics[width=\linewidth]{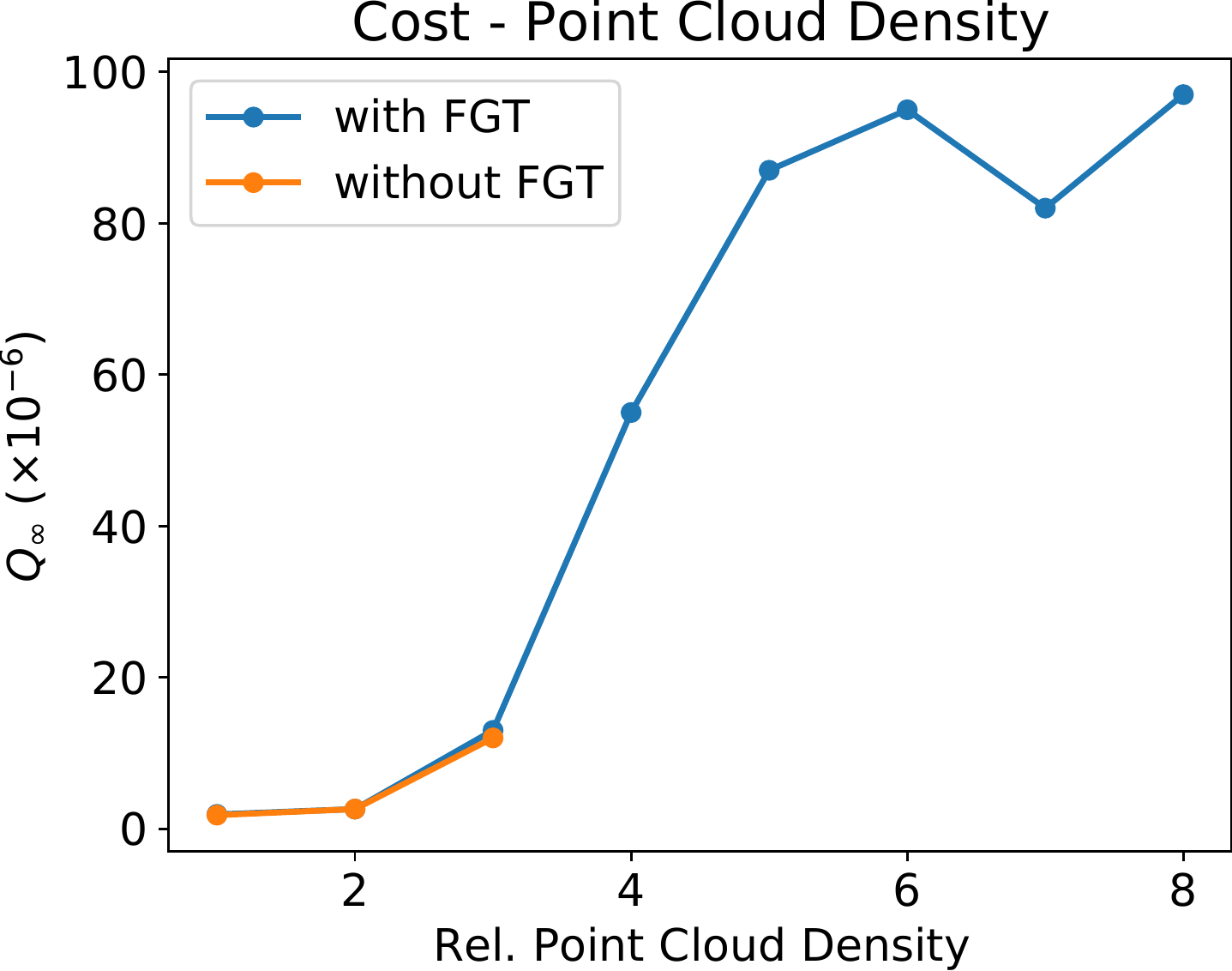}
\caption{\label{fig:fmmQuality} The averaged grasp quality plotted against the relative sample density. FMM acceleration generates almost identical $Q_\infty$ metric values as compared with brute-force summation.}
\vspace{-15px}
\end{figure}
\section{Conclusion \& Future Work}
We present a full-featured, model-based, differentiable grasp planner that can plan both precision and power grasps. We first establish the connection between grasp planning and contact-implicit path planning, which takes the form of an IPCC. We further show that IPCC can be rewritten as an NLP via the kernel-integral relaxation. Finally, we propose a SQP-based practical algorithm to solve the NLP, where the kernel-integral is approximately and efficiently evaluated using FMM. Our method achieves a higher level of generality in terms of 3D object types and gripper types, and we provide guaranteed (self-)collision-free results. In the future, we plan to apply our method to the training of robust, real-time grasp policies as in \cite{Liu2020DeepDG}. Our method can only find locally optimal grasps, and we plan to integrate our method with a stochastic global optimizer, such as Bayesian optimization \cite{nogueira2016unscented}, which can also handle uncertainties in object shapes.
%%%%%%%%%%%%%%%%%%%%%%%%%%%%%%%%%%%%%%%%%%%%%%%%%%%%%%%%%%%%%%%%%%%%%%%%%%%%%%%%
%\bibliographystyle{IEEEtranS}
%\bibliography{references}
\printbibliography
\section*{Appendix: SQP Optimizer}
We provide our main algorithm, which is a simplified, line-search-based SQP optimizer. We observe from \prettyref{eq:optCons} that we can always reduce $Q$ to satisfy the constraints $Q\leq G_d(\theta)$. Therefore, the underlying QP subproblem is always feasible and we do not need to use any feasibility relaxation. We assume the following exact $l_1$-merit function:
\begin{align}
\label{eq:merit}
\phi(\theta,r)=L(\theta)-Q+\rho\sum_d|\min(0,G_d(\theta)-Q)|,
\end{align}
and we assume following QP subproblem using approximate positive-definite Hessian $H$:
\begin{equation}
\begin{aligned}
\label{eq:QP}
\argmin{\Delta\theta,\Delta Q}\;&
\frac{1}{2}\Delta\theta^TH\Delta\theta+
\TWO{\Delta\theta^T}{\Delta Q^T}\TWOC{\FPP{L}{\theta}}{-1}\\
\ST\;&Q+\Delta Q\leq G_d(\theta)+\FPP{G_d}{\theta}\Delta\theta.
\end{aligned}
\end{equation}
This problem must be feasible using sufficient small $\Delta Q$. The size of matrix $H$ is small, typically less than $10\times10$, so we use eigen-decomposition and clamp the negative eigenvalues below $10^{-6}$ to ensure positive definiteness. The directional derivative of $\phi$ along $\TWO{\Delta\theta^T}{\Delta Q^T}$ is:
\begin{small}
\begin{align*}
D\phi(\theta,r)
%=&\TWO{\Delta\theta^T}{\Delta Q^T}\TWOC{\FPP{L}{\theta}}{-1}+
%\rho\sum_{d:G_d(\theta)-Q<0}\Delta Q-\FPP{G_d}{\theta}\Delta\theta\\
%\leq&\TWO{\Delta\theta^T}{\Delta Q^T}\TWOC{\FPP{L}{\theta}}{-1}+
%\rho\sum_{d:G_d(\theta)-Q<0}G_d(\theta)-Q\\
%=&
\leq\TWO{\Delta\theta^T}{\Delta Q^T}\TWOC{\FPP{L}{\theta}}{-1}-
\rho\sum_d|\min(0,G_d(\theta)-Q)|.
\end{align*}
\end{small}
To ensure that the directional derivative to be negative, we can choose:
\begin{equation}
\begin{aligned}
\label{eq:meritIncrease}
%&D\phi(\theta,Q)\leq-\gamma\rho\sum_d|\min(0,G_d(\theta)-Q)|\\
%\Rightarrow&
\rho\geq\frac{\TWO{\Delta\theta^T}{\Delta Q^T}\TWOC{\FPP{L}{\theta}}{-1}}
{(1-\gamma)\sum_d|\min(0,G_d(\theta)-Q)|}\quad\gamma\in(0,1).
\end{aligned}
\end{equation}
The final SQP algorithm for grasp planning is illustrated in \prettyref{alg:SQP}. Note that the optimization of the separating planes ${P^{pq}}^k$ are not included in the SQP framework. Instead, we update them in an alternating manner after each iteration. This treatment makes each iteration efficient and keep the Hessian matrix to have a small, fixed size. On the downside, the convergence speed degrades from second- to first-order, but the practical performance is satisfactory according to our experiments. 
\begin{algorithm}
\caption{\label{alg:SQP}SQP for Grasp Planning}
\begin{algorithmic}[1]
\Require{Initial $\theta^0, Q^0, {P^{pq}}^0, \gamma,\beta,c\in(0,1), \rho^0,\tau>0$}
\Ensure{Locally optimal $\theta$ to \prettyref{eq:optCons}}
\State $\phi^0\gets\phi(\theta^0,Q^0)$
\For{Iteration $k=1,2,\cdots$}
\State Use \prettyref{alg:FGT} to compute $G^d,\FPPR{G^d}{\theta}$
\State Solve \prettyref{eq:QP} for $\TWO{\Delta^k \theta}{\Delta^k Q}$
\State Increase $\rho^{k-1}$ to $\rho^k$ to ensure \prettyref{eq:meritIncrease}
\State $\Theta\gets1$\Comment{Line search}
\While{true}
\State $\theta^k\gets\theta^{k-1}+\Theta\Delta^k \theta$
\State $Q^k\gets Q^{k-1}+\Theta\Delta^k Q$
\State $\phi^k\gets \phi(\theta^k,Q^k)$
\If{$\phi^k\leq\phi^{k-1}+c\Theta D\phi(\theta^{k-1},Q^{k-1})$}
\State Break
\Else
\State $\Theta\gets\beta\Theta$
\EndIf
\EndWhile
\For{$1\leq p<q\leq L$}\Comment{Update separating plane}
\State ${P^{pq}}^k\gets
\argmin{R^{pq}\in\mathcal{SO}(3),n_0^{pq}}L_r(\theta^k,n^{pq},n_0^{pq})$
\EndFor
\If{$\left\|\FOURC{\Delta^k \theta}{\Delta^k Q}
{{n^{pq}}^k-{n^{pq}}^{k-1}}
{{n_0^{pq}}^k-{n_0^{pq}}^{k-1}}\right\|<\tau$}
\State Return $\theta^k$
\EndIf
\EndFor
\end{algorithmic}
\end{algorithm}
\end{document}